\documentclass[letterpaper]{article} 
\usepackage{aaai24}  
\usepackage{times}  
\usepackage{helvet}  
\usepackage{courier}  
\usepackage[hyphens]{url}  
\usepackage{graphicx} 
\usepackage{natbib}  
\usepackage{caption} 
\usepackage[linesnumbered, ruled]{algorithm2e}
\usepackage{cite}
\usepackage{float}
\usepackage{tikz}
\usepackage{amsmath}
\usepackage{amsfonts}
\usepackage{amssymb}
\usepackage{cite}
\usepackage{amsthm}
\usepackage{makecell}
\usepackage{booktabs}
\usepackage{tabularx}
\usepackage{subcaption}
\usepackage{soul}
\usepackage{newfloat}
\usepackage{listings}
\DeclareCaptionStyle{ruled}{labelfont=normalfont,labelsep=colon,strut=off} 
\floatstyle{ruled}
\newfloat{listing}{tb}{lst}{}
\floatname{listing}{Listing}

\newtheorem{definition}{Definition}

\newtheorem{proposition}{Proposition}

\newtheorem{formulation}{Problem Formulation}

\newcommand{\X}{\mathbf{X}}

\newcommand{\h}{\mathbf{h}}
\newcommand{\Z}{\mathbf{Z}}

\newcommand{\I}{\mathbf{I}}

\title{Long-Term Fair Decision Making through Deep Generative Models}
\author{
    Yaowei Hu\textsuperscript{\rm 1},
    Yongkai Wu\textsuperscript{\rm 2},
    Lu Zhang\textsuperscript{\rm 1}\\
}
\affiliations {
    \textsuperscript{\rm 1} University of Arkansas \\
    \textsuperscript{\rm 2} Clemson University \\
    yaoweihu@uark.edu, yongkaw@clemson.edu, lz006@uark.edu
}

\begin{document}

\maketitle

\begin{abstract}
This paper studies long-term fair machine learning which aims to mitigate group disparity over the long term in sequential decision-making systems.
To define long-term fairness, we leverage the temporal causal graph and use the 1-Wasserstein distance between the interventional distributions of different demographic groups at a sufficiently large time step as the quantitative metric. Then, we propose a three-phase learning framework where the decision model is trained on high-fidelity data generated by a deep generative model. We formulate the optimization problem as a performative risk minimization and adopt the repeated gradient descent algorithm for learning. The empirical evaluation shows the efficacy of the proposed method using both synthetic and semi-synthetic datasets.
\end{abstract}

 \section{Introduction}
Machine learning models are extensively utilized for significant decision-making scenarios, like college admissions \cite{baker2021algorithmic}, banking loans \cite{lee2021algorithmic}, job placements \cite{schumann2020we}, and evaluations of recidivism risks \cite{berk2021fairness}. 
Fair machine learning, which aims to reduce discrimination and bias in machine-automated decisions, is one of the keys to enabling broad societal acceptance of large-scale deployments of decision-making models. In the past decade, various notions, metrics, and techniques have emerged to address fairness in machine learning. For an overview of fair machine learning studies, readers are directed to recent surveys \cite{tang2022and,alves2023survey}.

However, our society is marked by pervasive group disparities. For example, in the context of bank loans, disparities in creditworthiness may be observed among different racial groups. These disparities can arise from systemic or historic factors that influence variables like credit score, employment history, and income. As another example, disparities may be observed in the qualification of students for college admission which could be influenced by factors such as the availability of higher education institutions in rural areas, the economic capacity of families, exposure to college preparatory resources, etc. 
Currently, the majority of studies in fair machine learning are focused on the problem of building decision models for fair one-shot decision-making \cite{mehrabi2021survey,caton2020fairness}.
However, it has been shown that algorithms based on traditional fairness notions, such as demographic parity (DP) and equal opportunity (EO), cannot efficiently mitigate group disparities and could even exacerbate the gap \cite{liu2018delayed,zhang2020fair}. 
The challenges lie in the sequential nature of real-world decision-making systems, where each decision may shift the underlying distribution of features or user behavior, which in turn affects the subsequent decisions. For example, a bank loan decision can have an impact on an individual's credit score, income, etc., which may influence his/her future loan application. As a result, long-term fairness has been proposed to focus on the mitigation of group disparities rather than making fair decisions in a single time step, which is more challenging to achieve than traditional fairness notions due to the feedback loops between the decisions and features as well as the data distribution shift during the sequential decision-making process. Please refer to the technical appendix for a literature review. 

In political and social science, affirmative action has been implemented as a practice to pursue the goal of group parity \cite{crosby2006understanding}. For example, for promoting group diversity, affirmative action has been used in higher education allowing universities to consider race as a factor in admissions. However, affirmative action is controversial and has been criticized on the grounds that race-conscious decisions might lead to reverse discrimination against other groups or stigmatize all members of the target group as unqualified. \cite{fischer2007effects}. Recently, the US Supreme Court ruled that affirmative action policies are unconstitutional and race can no longer be regarded as a factor in admissions to US universities \cite{harvard2023}. 

Given the context, we ask whether we can mitigate group disparity and achieve long-term fairness while limiting the use of the sensitive attribute in decision-making models. Although policies such as expanded outreach or "top X-percent" plans are worth investigating in practice, in this paper, we explore the data-driven prospect of mitigating group disparity only through learning and deploying appropriate decision-making models, taking into account the system dynamics. 
In other words, we aim to learn a decision model that, once deployed, can lead to group parity by properly reshaping the distribution.
To this end, we leverage Pearl's structural causal model (SCM) and causal graph \cite{pearl2009causality} for modeling the system dynamics and defining long-term fairness. Formally, we treat longitudinal features and decisions in the sequential decision-making process as temporal variables. Then, we model the system using temporal causal graphs \cite{pamfil2020dynotears} which are graphical representations that can depict causal relationships between features and decisions over time (see Figure \ref{fig:cgraph} for an example). The deployment of a decision-making model is simulated as soft interventions on the graph so that the influence of feedback can be inferred as the interventional distribution. Thus, group disparity produced by the decision-making model at any time step $T$ can be considered as the causal effect of the sensitive attribute on the features at time $T$. Meanwhile, sensitive attribute unconsciousness can be formulated as a local requirement that restricts the direct causal effect of the sensitive attribute on the decision at each time step. As a result, we show that our goal can be formulated as the trade-off between two conflicting causal effect objectives.

To strike a balance for the trade-off, we leverage deep generative models to devise a regularized learning problem. Specifically, given the observed time series within the time range $[1,l]$, we train a deep generative model to generate the interventional distribution for the range $[1,T]$ as well as the observational distribution if $T>l$. Then, we integrate the deep generative model and decision-making model into a collaborative training framework so that the predicted data could be used as reliable data for training the decision model. We propose a three-phase training framework. In Phase 1, given the training time series, we first train a base decision model without considering any fairness requirements. In Phase 2, we train a recurrent conditional generative adversarial network (RCGAN) inspired by \cite{esteban2017real} to fit the training time series in order to generate high-fidelity time series. Finally, in Phase 3, we train a fair decision model on the generated time series within the time range $[1,T]$ by considering both long-term fairness at time $T$ as well as the local fairness requirement at each time step.
The optimization problem is treated as performative risk minimization and solved by using the repeated gradient descent algorithm \cite{perdomo2020performative}.

To derive the quantitative long-term fairness metric, we consider the interventional distribution of features at time step $T$ and adopt the 1-Wasserstein distance between the interventional distributions of different demographic groups as the measure of the group disparity. We substantiate our metric by demonstrating that a small distance enables reconciliation to occur between DP and EO for any decision models that are unconscious of sensitive attributes, provided the necessary conditions from \cite{kim2020fact} are met, and hence implying systemic equality and group parity.

Our experiments demonstrate that the proposed framework can effectively balance various fairness requirements and utility, while methods based on traditional fairness notions cannot effectively achieve long-term fairness.

\section{Background}
Throughout this paper, single variables and their values are denoted by the uppercase letter $X$ and the lowercase letter $x$. Corresponding bold letters $\mathbf{X}$ and $\mathbf{x}$ denote the sets of variables and their values respectively. We utilize Pearl's structural causal model (SCM) and causal graph \cite{pearl2009causality} for defining the long-term fairness metric and designing the architecture of the deep generative model. For a gentle introduction to SCM please refer to \cite{pearl2010causal}. In this paper, we assume the \textit{Markovian SCM} such that the exogenous variables are mutually independent.

Causal inference in the SCM is facilitated with the interventions \cite{pearl2009causality}. The hard intervention forces some variables to take certain constant. The soft intervention, on the other hand, forces some variables to take certain functional relationship in responding to some other
variables \cite{correa2020calculus}. Symbolicly, the soft intervention that substitutes equation $X = f_X(\mathbf{Pa}_X, \mathbf{U}_X)$ with a new equation $X = g(\mathbf{Z})$ is denoted as $\sigma_{X = g(\mathbf{Z})}$. The distribution of another variable $Y$ after performing the soft intervention is denoted as $P(Y(\sigma_{X = g(\mathbf{Z})}))$.

\section{Formulating Long-term Fairness}
\subsection{Problem Setting}
We start by modeling the system dynamics and formulating a long-term fairness metric that permits continuous optimization.
To ease the representation, we assume a binary sensitive attribute for indicating different demographic groups denoted by $S \in \mathcal{S} = \{s^+, s^-\}$, as well as a binary decision denoted by $Y \in \mathcal{Y} = \{y^+, y^-\}$. The profile features other than the sensitive feature are denoted by $\mathbf{X}\in \mathcal{X}$. In a sequential decision-making system, if a feature is time-dependent, it means that its value may change from one time step to another. We assume that $\X$ and $Y$ are time-dependent and use the superscript to denote their variants at different time steps, leading to $\X^t$ and $Y^t$. Many sensitive attributes like gender and race are naturally time-independent. 
Some sensitive attributes may change over time, but the relative order of individuals in the data does not change, like age. Thus, we treat $S$ as being time-independent in this paper.

Suppose that we have access to a time series $\mathcal{D}$ = \{$(S, \mathbf{X}^t, Y^t)\}_{t=1}^{l}$. 
We assume an SCM for describing the data generation mechanism and leverage a temporal causal graph for describing the causal relation among $S,\X^t,Y^t$ in the SCM. 
We make the \textit{stationarity assumption} such that data distribution may shift over time but the data generation mechanism behind it does not change. 
Figure \ref{fig:cgraph} gives an example which shows that at each time step the decision $Y^{t}$ is made based on the value of $\X^{t}$ and $S$. Meanwhile, the value of $\X^t$ is affected by the values of $\X^{t-1}$, $Y^{t-1}$ and $S$. We will use this graph as a running example throughout the remaining of this paper. In practice, the temporal causal graph can be obtained from the domain knowledge or learned from data using structure learning algorithms (e.g., \cite{vowels2021d,runge2020discovering,pamfil2020dynotears}). 
Our goal is to learn a decision model $h_{\theta}: \mathcal{S} \times \mathcal{X} \mapsto \mathcal{Y}$ such that when deployed at every time step, certain fairness requirements can be achieved.

\begin{figure}[t]
	\centering
	\begin{tikzpicture}[scale=1.]
	\definecolor{blue}{rgb}{0.0, 0.0, 0.0}
	\definecolor{red}{rgb}{0.9, 0.17, 0.31}
	\definecolor{green}{rgb}{0.0, 0.5, 0.0}
	
	\tikzstyle{vertex} = [thick, circle, draw, inner sep=0pt, minimum size = 6mm]
	\tikzstyle{empty} = [thick, circle, inner sep=0pt, minimum size = 4mm]
	\tikzstyle{rect} = [thick, rectangle, draw, inner sep=0pt, minimum size = 5mm]
	\tikzstyle{remp} = [thick, rectangle, inner sep=0pt, minimum size = 3mm]
	\tikzstyle{edge} = [thick, -stealth, blue]
	
	\node[rect] at (0, 3.2) (S) {\tiny{$S$}};
	
	\node[vertex] at (0.0,2) (X0) {\tiny{$X^1$}};
	\node[vertex] at (1.5,2) (X1) {\tiny{$X^2$}};
	\node[vertex] at (3.0,2) (X2) {\tiny{$X^3$}};
	\node[empty]  at (4.2,2) (X3) {$...$};
	\node[vertex] at (5.4,2) (X*) {\tiny{$X^T$}};
	
	\node[vertex] at (0.0, 0.8) (Y0) {\tiny{$Y^1$}};
	\node[vertex] at (1.5, 0.8) (Y1) {\tiny{$Y^2$}};
	\node[vertex] at (3.0, 0.8) (Y2) {\tiny{$Y^3$}};
	\node[empty]  at (4.2, 0.8) (Y3) {$...$};
	\node[vertex] at (5.4, 0.8) (Y*) {\tiny{$Y^{T}$}};
	
    \draw[edge, ] (S) to (X0);
    \draw[edge, ] (S) .. controls (0.5, 3.2) and (1.2, 3.) .. (X1);
    \draw[edge, ] (S) .. controls (1.5, 3.2) and (2.2, 3.) .. (X2);
    \draw[edge, ] (S) .. controls (2.5, 3.2) and (4.2, 3.) .. (X*);
    
    \draw[edge, ] (X0) to (X1);
    \draw[edge, ] (X1) to (X2);
    \draw[edge, ] (X2) to (X3);
    \draw[edge, ] (X3) to (X*);
    
    \draw[edge, ] (X0) to (Y0);
    \draw[edge, ] (X1) to (Y1);
    \draw[edge, ] (X2) to (Y2);
    \draw[edge, ] (X*) to (Y*);
    
    \draw[edge, ] (Y0) to (X1);
    \draw[edge, ] (Y1) to (X2);
    \draw[edge, ] (Y2) to (4.0, 1.8);
    \draw[edge, ] (4.5, 1.2) to (X*);
    
    \draw[edge] (S) .. controls (-0.8, 2.8) and (-0.8, 1.4) .. (Y0);
    \draw[edge] (S) to (Y1);
    \draw[edge] (S) .. controls (1.3, 3.2) and (1.8, 2.8) .. (Y2);
    \draw[edge] (S) .. controls (2.5, 3.2) and (4.2, 3.) .. (Y*);

	\end{tikzpicture}
	\caption{A temporal causal graph for sequential decision making.}
	\label{fig:cgraph}
\end{figure}
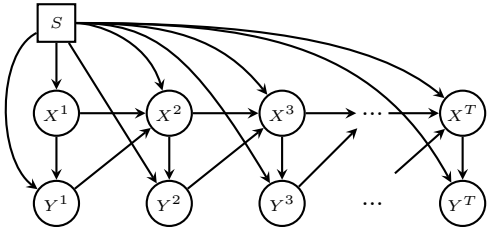

To illustrate our problem setting in a real-world scenario, consider an example of a bank loan system. When people apply for bank loans, their personal information (e.g., race, job, assets, credit score, etc.) is used by the bank's decision model to decide whether to grant the loans. Except for race, which is a sensitive feature $S$, other profile features $\X^{t}$ represent an applicant's qualification at time step $t$. The bank's decision $Y^{t}$ can have impacts on the applicants' profile features in $\textbf{X}^{t+1}$ such as the credit score, which in turn affect the outcomes of their subsequent loans.

\subsection{Long-term Fairness Metric}
To formulate long-term fairness under the decision model deployment, \cite{hu2022achieving} has proposed to mathematically simulate the model deployment by soft interventions on $Y$ at all time steps. That is, given a decision model $h_{\theta}$, it replaces the original structural equation of $Y$ in the SCM. We denote the soft intervention by $\sigma_{Y^{t}=h_{\theta}(S,\X^{t})}$ and abbreviate it as $\sigma_{\theta}$.
Then, the influence of the model deployment on feature $\X^{t}$ can be described by its interventional distribution, denoted by $P(\X^{t}(\sigma_{\theta}))$. With this formulation, we treat group disparity at time step $T$ as the causal effect of $S$ on $\mathbf{X}^T$, i.e., group parity is achieved when there is no such causal effect and the interventional distribution of $\X^{T}$ will be equal across demographic groups $\{s^+, s^-\}$. 

It is worth noting that, the causal effect of $S$ on $\X^T$ is transmitted through multiple causal paths. These paths can be categorized into two sets: those that intersect with decision nodes on the graph (e.g., $S\rightarrow \X^1 \rightarrow Y^1 \rightarrow \X^2 \rightarrow \cdots$), and those that bypass these nodes (e.g., $S\rightarrow \X^1 \rightarrow \X^2 \rightarrow \cdots$). The causal effect transmitted through the former set is influenced by updates to the decision model, whereas the effect via the latter remains unchanged. Hence, eliminating the causal effect of $S$ on $\X^T$ via updating the decision model means learning the decision model $h_{\theta}$ such that the causal effects transmitted through the two sets of paths are cancelled out. Consequently, we can formulate the problem of mitigating group disparity as a learning problem. 

We note that the causal effect transmitted through $h_{\theta}$ will be limited by the requirement of sensitive attribute unconsciousness that restricts the direct causal effect of $S$ on $h_{\theta}$ at each time step. The formulation of the metric of sensitive attribute unconsciousness will be detailed in the next section. As a result, long-term fairness may not always be achievable only through updating the decision model. In this paper, we aim to develop and conduct the best practices for mitigating group disparity using learning algorithms. The theoretical analysis of the feasibility of long-term fairness will be a valuable direction for future research.

Next, we present the quantitative metric of long-term fairness. We use the 1-Wasserstein distance to measure the difference between the two interventional distributions, defined as follows.
\begin{definition}\label{def:lf}
Given a sequential decision making system, a decision model $h_{\theta}: \mathcal{S} \times \mathcal{X} \mapsto \mathcal{Y}$, and a time step $T$, the metric for measuring the long-term fairness produced by deploying $h_{\theta}$ is given by 
\begin{equation}\label{eq:jlt}
    J_{1}^{T}(\theta) \triangleq W(P(\X^{T}(\sigma_{\theta})|S=s^+),P(\X^{T}(\sigma_{\theta})|S=s^-)),
\end{equation}
where $W$ is the 1-Wasserstein distance and $\sigma_{\theta}$ is the soft intervention.
\end{definition}

We substantiate our metric and the choice of the 1-Wasserstein distance with the following proposition.

\begin{proposition}\label{pro:wd}
Let $d$ be the 1-Wasserstein distance given in Definition \ref{def:lf}.
For any sensitive attribute-unconscious decision model $f: \mathcal{X} \mapsto \mathcal{A}$ that is Lipschitz continuous, its DP is bounded by $l_f \cdot d$ where $l_f$ is the Lipschitz constant of $f$. If we assume that the true label $Y$ is given by a decision model $g: \mathcal{X} \mapsto \mathcal{A}$ that is Lipschitz continuous and satisfies the equal base rate condition, then the EO of $f$ is bounded by $(l_f+l_g)/P(y)\cdot d$ where $l_g$ is the Lipschitz constant of $g$.
\end{proposition}

Please refer to the technical appendix for the proof. Proposition \ref{pro:wd} implies that when the 1-Wasserstein distance is minimized, both DP and EO are mitigated at time $T$ for any sensitive attribute-unconscious decision model that is Lipschitz continuous. As shown in \cite{kim2020fact}, the equal base rate condition is the necessary condition for DP and EO to be compatible. Proposition \ref{pro:wd} means that minimizing the 1-Wasserstein distance transforms the necessary condition into both a necessary and sufficient condition. This demonstrates a reconciliation between the previously incompatible fairness notions of DP and EO, thereby suggesting group parity and system equity.

In the experiments, to reduce computational complexity, we use the Sinkhorn distance \cite{cuturi2013sinkhorn} to approximate the 1-Wasserstein distance.
Note that although Proposition \ref{pro:wd} applies for sensitive attribute-unconscious decision models, we do allow the decision model $h_{\theta}$ to take the sensitive feature as input. Meanwhile, we impose local fairness constraints to restrict the direct causal effect of $S$ on the decision, as shall be shown in the next section.

\section{Learning Framework for Long-term Fairness}
In this section, we introduce the learning framework for achieving long-term fairness.

\subsection{Problem Formulation}

We start by formulating the objective function for the decision model $h_{\theta}$. In addition to long-term fairness outlined in Definition \ref{def:lf}, additional factors must also be taken into account during optimization. The first factor is the utility of the decision model to ensure good prediction performance. We adopt the traditional definitions of the loss function given as follows. In practice, loss functions such as cross-entropy loss can be used to penalize inaccurate predictions.

\begin{definition}
Given a time series $\mathcal{D}$ = $\{(S, \mathbf{X}^t, Y^t)\}_{t=1}^{l}$, the loss for the decision model $h_{\theta}$ is given by
\begin{equation}\label{eq:ju}
    J_{2}(\theta) \triangleq \frac{1}{l} \sum_{t=1}^{l} \mathbb{E}[\mathcal{L}(h_{\theta}(S,\X^{t}),Y^{t})],
\end{equation}
where $\mathcal{L}$ is any loss function.
\end{definition}

The second factor is the local fairness constraint for ensuring sensitive attribute-unconsciousness as mentioned above. In this paper, we adopt direct discrimination proposed in \cite{zhang2016causal} as the metric for formulating the local fairness constraint, as defined below.
\begin{definition}
The local fairness constraint for each time step $t\in [1,T]$ is given by direct discrimination at $t$, i.e.,
\begin{equation}\label{eq:jst}
    \begin{split}
    J_{3}^{t}(\theta) \; \triangleq \; & | \mathbb{E}[h_{\theta}(S=s^+,\X^{t}(\sigma_{\theta}))|S=s^-] - \\ &\mathbb{E}[h_{\theta}(S=s^-,\X^{t}(\sigma_{\theta}))|S=s^-] |.
    \end{split}
\end{equation}
\end{definition}

We note the trade-off between the three factors. The local fairness and utility trade-off has been studied in previous works \cite{bakker2019fairness,corbett2017algorithmic}. What we are more interested is the trade-off between long-term and local fairness. As discussed, mitigating group disparity requires a non-zero causal effect to be transmitted through decision nodes so that the causal effects transmitted through the two sets of paths can be canceled out. However, sensitive attribute unconsciousness requires eliminating the direct causal effect to be transmitted through decision nodes. As a concrete example, making loan decisions exactly according
to the credit score is reasonable in terms of local fairness, but might not help in narrowing the gap in credit scores between advantaged and disadvantaged groups. On the other hand, reducing the gap by favoring the
disadvantaged group can raise the potential issue of reverse discrimination, which may violate the local fairness constraint.

In this paper, rather than delving into a theoretical analysis of the trade-off between long-term and local fairness, we utilize the learning algorithm to achieve a balance between them. We formulate a regularized learning problem, leading to the problem formulation below.

\begin{formulation}
The problem of long-term fair decision-making is to solve the optimization problem:
\begin{equation}\label{eq:obj}
    \min_{\theta} \mathcal{L}(\theta)= \lambda_{1}J_{1}^{T}(\theta) + \lambda_{2}J_{2}(\theta) + \frac{\lambda_{3}}{T}\sum_{t=1}^{T} J_{3}^{t}(\theta),
\end{equation}
where $\lambda_{1}$, $\lambda_{2}$ and $\lambda_{3}$ are weight parameters.
\end{formulation}

\subsection{Overview of Learning Framework}
There are two main challenges in solving Problem Formulation 1. First, as shown in Eqs.~\eqref{eq:jlt} and \eqref{eq:jst}, the computation of $J_{1}^{T}(\theta)$ and $J_{3}^{t}(\theta)$ are based on the interventional variants of features $\X^{t}(\sigma_{\theta})$ whose values in turn depend on the model parameter $\theta$. Second, if $T>l$, then $J_{1}^{T}(\theta)$ and $\sum_{t=1}^{T} J_{3}^{t}(\theta)$ will be computed on time steps that are beyond the range of the training data.
Thus, Problem Formulation 1 cannot be solved by traditional machine learning algorithms.

We propose a novel three-phase framework based on causal inference techniques and deep generative networks. The core idea is to use a deep generative network to simulate an SCM for generating both observational and interventional distributions. Our method rests on two theoretical foundations: 1) \cite{kocaoglu2018causalgan} shows that if the structure of a generative network is arranged to reflect the causal structure, then it can be trained with the observational data such that it will agree with the same SCM in terms of any identifiable interventional distributions; 2) \cite{kocaoglu2019characterization} shows that the interventional distribution produced by any soft intervention is identifiable in a Markovian SCM.

Inspired by these prior results, we design the architecture of the deep generative network following the causal structure. We illustrate our framework using the example shown in Figure \ref{fig:cgraph}. In this example, the causal structure at each time step can be mathematically described by two structural equations of the SCM:
\begin{align}
     Y^t &=   f_Y\left(S, \mathbf{X}^t, U_{Y} \right) \label{eq:fy},\\ 
    \mathbf{X}^t &=  f_\mathbf{X}\left(S, \mathbf{X}^{t-1}, Y^{t-1}, U_\mathbf{X}\right),
    \label{eq:fx}
\end{align}
where $U_Y,U_{\mathbf{X}}$ are exogenous variables.
Note that due to the stationarity assumption, $f_Y$ and $f_{\mathbf{X}}$ are time-independent.
According to the principle of independent mechanisms \cite{peters2017elements}, these two structural equations can be learned independently without affecting each other.
Motivated by this principle, we propose a three-phase framework. In Phase 1, we train a classifier $h_{\omega}$ on the training time series $\mathcal{D}$ to approximate $f_Y$ and make decisions for each time step. In Phase 2, we train a recurrent conditional generative adversarial network (RCGAN) \cite{esteban2017real} on the same training time series $\mathcal{D}$ using adversarial training. The generator of the RCGAN uses $h_{\omega}$ obtained in Phase 1 for generating decisions. Finally, in Phase 3, we replace $h_{\omega}$ with the decision model $h_{\theta}$ and train it on the data generated from the RCGAN using the objective function Eq.~\eqref{eq:obj}. The overview of the framework is shown in Figure \ref{fig:process} and the pseudo-code is shown in Algorithm \ref{algo}. Next, we describe the details of each phase.

\begin{figure}[t]
    \centering
    \includegraphics[width=0.47\textwidth]{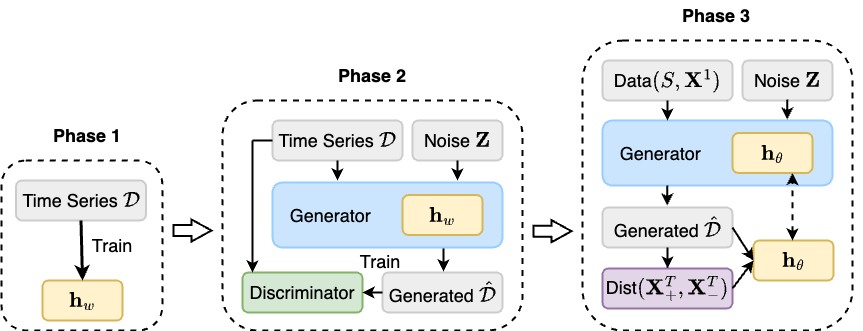}
    \caption{The overview of the proposed framework. Solid arrows represent input, and the dashed arrow represents parameter sharing. For Phase 3 only one generator is shown.}
    \label{fig:process}
\end{figure}

\IncMargin{1em}
\begin{algorithm}[t]\small
\SetKwInOut{Input}{Input}
\SetKwInOut{Output}{Output}
\SetKwRepeat{Repeat}{repeat}{until}
\SetKw{Return}{return}
	\caption{\textbf{DeepLF}}
    \label{algo} 
	\Input{Dataset $\mathcal{D} = \{(S, \mathbf{X}^t, Y^t)\}_{t=1}^{l}$, time-lagged causal graph $\mathcal{G}$, parameters $\lambda_1$, $\lambda_2$ and $\lambda_3$}
	\Output{The fair model $h_{\theta}$}
	\BlankLine 
	
	Train a classifier $h_{\omega}$ by minimizing Eq.~(8) on $\mathcal{D}$\; 
	\Repeat {\rm{convergence}} {
	    Update the discriminator $D_{\phi}$ according to Eq.~(13)\;
	    Update the generator $G_{\psi,\omega}$ with the classifier $h_{\omega}$ as one of its components according to Eq.~(14)\;
	}
	$i\leftarrow 0$\;
	Initialize $h_{\theta_{0}}$ according to $h_{\omega}$\;
	\Repeat{\rm{convergence}}{
	    Generate time series using generator $G_{\psi,\omega}$\;
	    Compute $\nabla_{\theta}\mathcal{L}_{l}(\theta)$ according to Eq.~(5) using the generated data\;
	    $\theta_{i+1}\leftarrow \theta_i - \eta_i \nabla_{\theta}\mathcal{L}_{l}(\theta_i)$, $i\leftarrow i+1$\;
	}
    \Return $h_{\theta_{i}}$;
 	 	  
\end{algorithm}
\DecMargin{1em}

\subsection{Phase 1: Train a Decision Classifier}
The objective of this phase is to learn a classifier $h_{\omega}$ from the training time series to approximate the mechanism $f_Y$ in Eq.~\eqref{eq:fy} for making decisions, i.e., $\hat{Y}^{t}=h_{\omega}(S,\X^{t})$. In this phase, we train the classifier by maximizing the accuracy without considering any fairness requirement. 
Since $f_Y$ does not change with time, we aggregate the losses at all time steps in the loss function. Specifically, if the cross-entropy loss is used, then the loss function for training $h_{\omega}$ on time series $\mathcal{D}$ is given by

\begin{equation}
    \min_{\omega} \mathcal{L}_{c}(\omega) = -\frac{1}{l} \sum_{t=1}^{l}\mathbb{E} \left[Y^t \log h_{\omega}(S, X^{t}) \right].
\end{equation}

\subsection{Phase 2: Train an RCGAN}
\label{sec:gan}
We then train an RCGAN to simulate the system dynamics and generate values of features $\X$ by taking the predictions of $h_{\omega}$ as the input. Our purpose is to generate both observational and intervention distributions of $\X$, so it is important to design the architecture of the RCGAN following the causal structure.

\begin{figure}[t]
    \centering
    \includegraphics[width=0.45\textwidth]{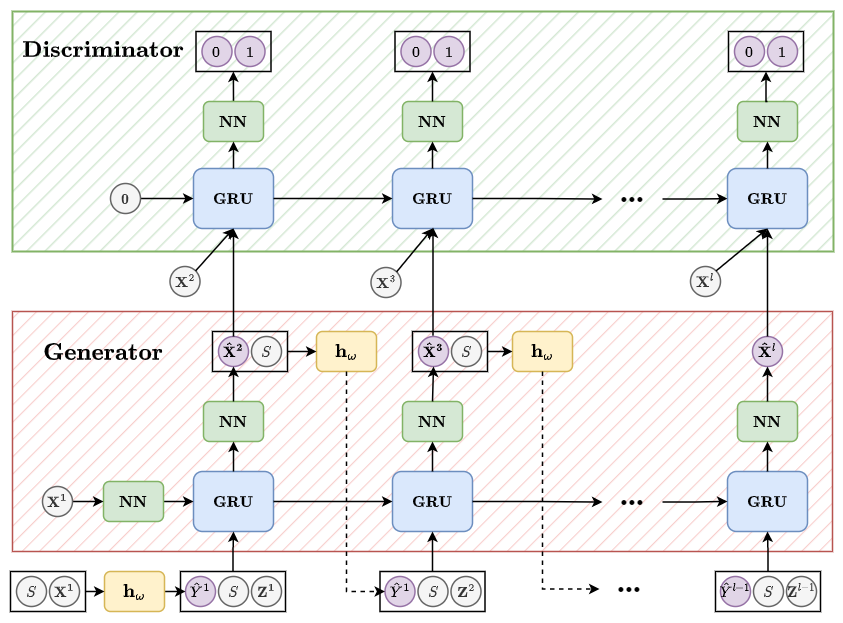}
    \caption{The architecture of the RCGAN.}
    \label{fig:gan}
\end{figure}

The architecture of the RCGAN is illustrated in Figure \ref{fig:gan}, which consists of one generator and one discriminator. We use the gated recurrent unit (GRU) \cite{cho2014properties} as the core structure of the generator and discriminator. For the generator, it takes the sensitive feature $S$, a set of noise vectors $\Z$, as well as the features at the first time step $\X^{1}$ as the input.
The hidden state is then initialized by a non-linear transformation (e.g., a multi-layer perceptron) of $\X^{1}$ as:
\begin{equation}
    \h^1 = \text{MLP}\left(\X^1\right).
\end{equation}
Then, for each time step $t$, we concatenate the noise vector $\Z^{t-1}$ with the conditional information $\hat{Y}^{t-1}$ and $S$ as the input to each GRU. After the calculation of GRU, 
we again use a non-linear transformation to convert the hidden states $\h^t$ to predicted $\hat{\X}^{t+1}$. These three steps are shown by the three equations below.
\begin{align}
    &\I^{t-1} \gets \left[\hat{Y}^{t-1}, S, \Z^{t-1}\right], \\
    &\h^t = \text{GRU}\left(\I^{t-1}, \h^{t-1}\right), \\ 
    &\hat{\X}^{t+1} = \text{MLP}\left(\h^t\right).
\end{align}

We also use GRUs for the discriminator to distinguish between generated data and real data.
For the generated data, the discriminator attempts to predict label 0 for each time step, and vice versa, for the real data, the discriminator attempts to predict label 1 for each time step.

Finally, for training the generator and discriminator, in addition to the objective of the original GAN for minimizing the likelihood of generated data given by the discriminator, the maximum mean discrepancy (MMD) \cite{gretton2006kernel} between original data and generated data is also explicitly minimized. The MMD brings two distributions together by comparing their statistics. As a result, the loss functions of the discriminator and the generator are shown below, which are optimized alternatively.
\begin{equation}
    \label{eq:discriminator}
    \begin{split}
        \max_{\phi} \mathcal{L}_d(\phi) = & \mathbb{E}_{\X} [\log(D_{\phi}(\X))] + \\ 
        & \mathbb{E}_{\Z} [\log(1-D_{\phi}(G_{\psi,\omega}(S,\Z,\X^{1})))],
    \end{split}
\end{equation}

\begin{equation}
    \label{eq:generator}
    \begin{split}
        \min_{\psi} \mathcal{L}_g(\psi) = &\mathbb{E}_{\Z} [\log(1-D_{\phi}(G_{\psi,\omega}(S,\Z,\X^{1})))] + \\
        & \gamma \text{MMD}(\X, G_{\psi,\omega}(S,\Z,\X^{1})),
    \end{split}
\end{equation}
where $D_{\phi}$ represents the discriminator, $G_{\psi,\omega}$ represents the generator that uses $h_{\omega}$ as the classifier, and $\gamma$ controls the strength of the regularization.

\subsection{Phase 3: Train a Long-term Fair Decision Model}
In the last phase, we train a decision model $h_{\theta}$ on the data generated by the RCGAN using the objective function Eq.~\eqref{eq:obj}. We use the generator obtained in Phase 2 as well as a variant of this generator where $h_{\omega}$ is replaced with $h_{\theta}$. The former is for generating the observational distribution and the latter is for generating interventional distribution. Specifically, we first directly apply the generator $G_{\psi,\omega}$ obtained in Phase 2 to generate data for time steps from 1 to $l$ which are used to compute $J_2(\theta)$ in Eq.~\eqref{eq:ju}. Then, we perform soft intervention $\sigma_{\theta}$ by replacing $h_{\omega}$ with $h_{\theta}$ to obtain a variant generator $G_{\psi,\theta}$ which is used to generate data for time steps from 1 to $T$ for computing $J_1^T(\theta)$ in Eq.~\eqref{eq:jlt} and $J_3^t(\theta)$ in Eq.~\eqref{eq:jst}. In other words, we use $G_{\psi,\theta}$ to generate the interventional data $\X^{t}(\sigma_{\theta})$. Finally, the loss $\mathcal{L}(\theta)$ is computed and $h_{\theta}$ is updated accordingly. Note that the RCGAN trained in Phase 2 will not be updated in this phase.

It is important to note that when we use the RCGAN to generate data samples for computing $\mathcal{L}(\theta)$, those data samples are affected by $h_{\theta}$ as well, due to the fact that $h_{\theta}$ is trained on the interventional distribution after performing soft intervention $\sigma_{\theta}$. This optimization problem is different from the traditional empirical risk minimization and is called the performative risk minimization \cite{perdomo2020performative}. In our work, we adopt the repeated gradient descent algorithm (RGD) \cite{perdomo2020performative} which is an iterative training approach to address this problem. In the training process of Phase 3, we first initialize $h_{\theta}$ according to $h_{\omega}$. Then, in each iteration, we use the current version of $h_{\theta}$ for generating data and computing the empirical loss, and $h_{\theta}$ is updated based on the empirical gradient $\nabla_{\theta}\mathcal{L}(\theta)$. After that, we replace $h_{\theta}$ with its updated version and conduct another iteration of training. This process is repeated until the parameters of $h_{\theta}$ converge.

\section{Experiments}
In this section, we conduct empirical evaluations of our method\footnote{The code and hyperparameter settings are available online: \url{https://github.com/yaoweihu/Generative-Models-for-Fairness}.}. We refer to our method as deep long-term fair decision making (\textbf{DeepLF}). We use a multi-layer perceptron for both $h_{\omega}$ and $h_{\theta}$ in our method.

\subsection{Datasets}
Many commonly used datasets in fair machine learning \cite{le2022survey} are not for dynamic fairness research. In \cite{ding2021retiring}, the authors construct a dataset that spans multiple years and allows researchers to study temporal shifts in the distribution level. However, our study requires the longitudinal data that track each instance over time, other than multiple datasets with temporal distribution shifts.  
Thus, following \cite{hu2022achieving}, we generate synthetic and semi-synthetic time series datasets as follows.

{\bf\noindent Synthetic Dataset.} We generate the synthetic time series dataset based on the temporal causal graph shown in Figure \ref{fig:cgraph}.
Each sample at each time step in the time series includes a sensitive feature $S$, profile features $\textbf{X}^t$ and a decision $Y^t$. The samples at the initial time step $\textbf{X}^1,Y^{1}$ are generated by calling the data generation function (i.e., make\_classification) of scikit-learn package. Then, we cluster the generated samples into two groups and assign $S$ to each sample according to the cluster it belongs to. To generate the data samples in the remaining time steps, we design a procedure by simulating the bank loan system in the real world. Please refer to the technical appendix for the detailed generation rules, following which we generate a 10-step synthetic time series dataset with 10000 instances where $\textbf{X}^t$ is a 6-dimensional vector. We refer to this dataset SimLoan.

{\bf\noindent Semi-Synthetic Dataset.}
We also generate semi-synthetic data by leveraging the real-world Taiwan dataset \cite{yeh2009comparisons} as the initial data at $t=1$. A ground-truth classifier and similar generation rules of change are used to generate subsequent decisions $Y^1, ..., Y^l$ and profile features $\textbf{X}^2, ..., \textbf{X}^l$. There are 10000 instances in the initial data and they are randomly and equally sampled from groups by $S$ and $Y$ for balance. Like the SimLoan dataset, this dataset is also made up of 10 steps. We refer to this dataset Taiwan.

\subsection{Experimental Setting}
{\bf\noindent Baselines.} A multi-layer perceptron with the same number of layers as $h_{\omega}$ and $h_{\theta}$ that is trained on the training time series $\mathcal{D}$ without any fairness constraints is used as the first baseline, denoted as (\textbf{MLP}). Two common static fairness constraints, i.e., demographic parity and equal opportunity, are applied to the \textbf{MLP} model as fairness constraints respectively, referred to as \textbf{MLP-DP} and \textbf{MLP-EO}.  We also implement the method proposed in \cite{hu2022achieving} that formulates long-term fairness as path-specific effects and trains the model using repeated risk minimization, referred to as \textbf{LRLF}. 
It requires the true causal structure equations for training, so we provide it with the true data generation rules.
Implementing details are included in the appendix.

{\bf\noindent Evaluation.}
To evaluate the performance of models after deployment, the RCGAN trained in Phase 2 and the decision models that we evaluate are used together to generate interventional data on which the local and long-term fairness are computed. The long-term fairness is measured by Eq.~\eqref{eq:jlt} computed on the evaluated decision model; the local fairness is measured by direct discrimination Eq.~\eqref{eq:jst} at each time step;
and the accuracy of predictions is evaluated based on the ground-truth classifier $h_{\omega}$ at each time step.

\subsection{Results}

\begin{figure*}[t!]
  \centering
  \includegraphics[scale=0.335]{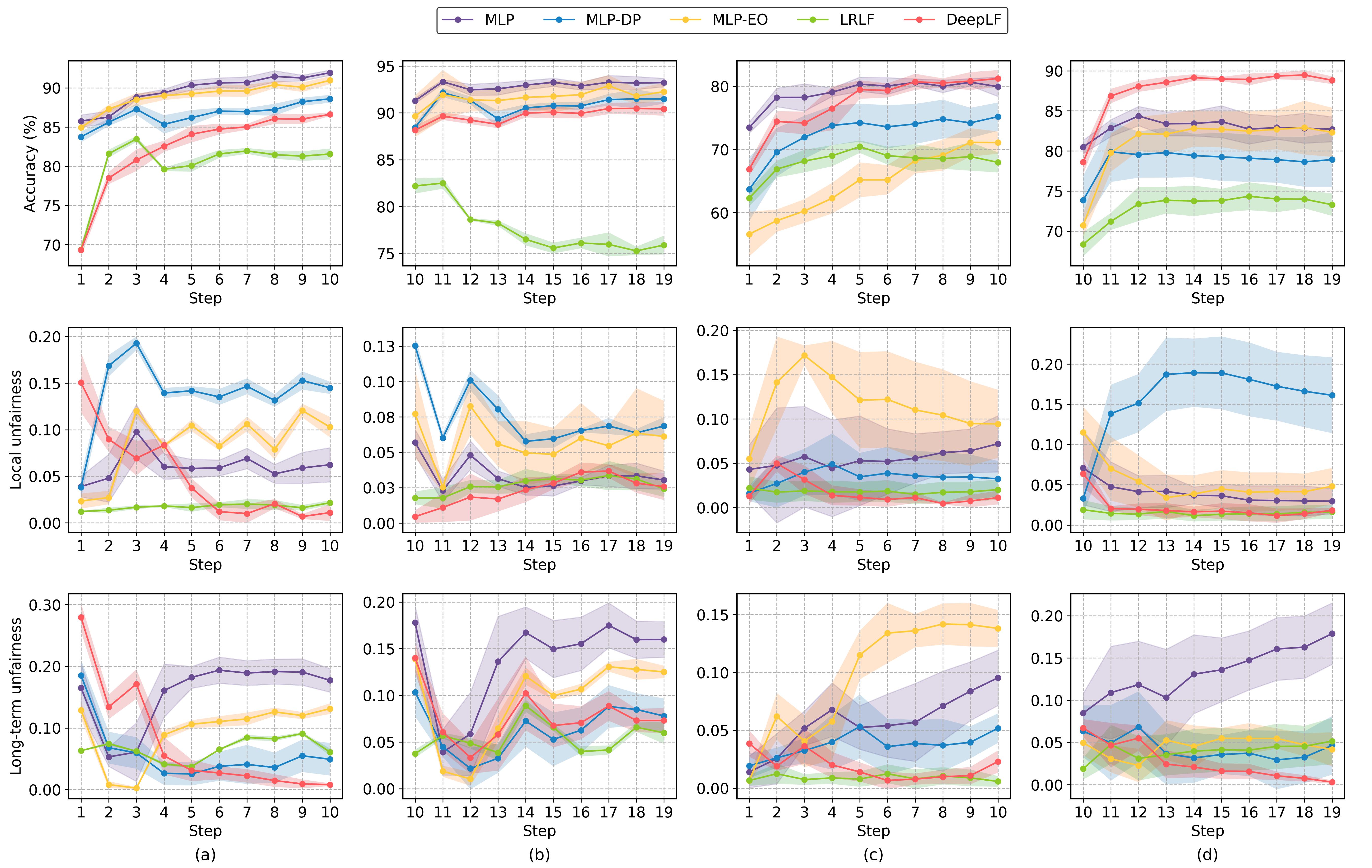}
  \caption{Accuracy ($\uparrow$), local and long-term unfairness ($\downarrow$) of different algorithms on SimLoan ((a) and (b)) and Taiwan ((c) and (d)) datasets. The decision models are trained on generated data within the time range $[1, 10]$. (a) and (c): Results of evaluation on generated data within time range $[1, 10]$. (b) and (d): Results of evaluation on generated data within the time range $[10, 19]$.}
  \label{fig:result}
\end{figure*}

\begin{figure}[!t]
    \centering
    
    \subfloat[SimLoan Dataset]{
    \begin{minipage}[c]{0.45\textwidth}
        \centering
        \includegraphics[scale=0.24]{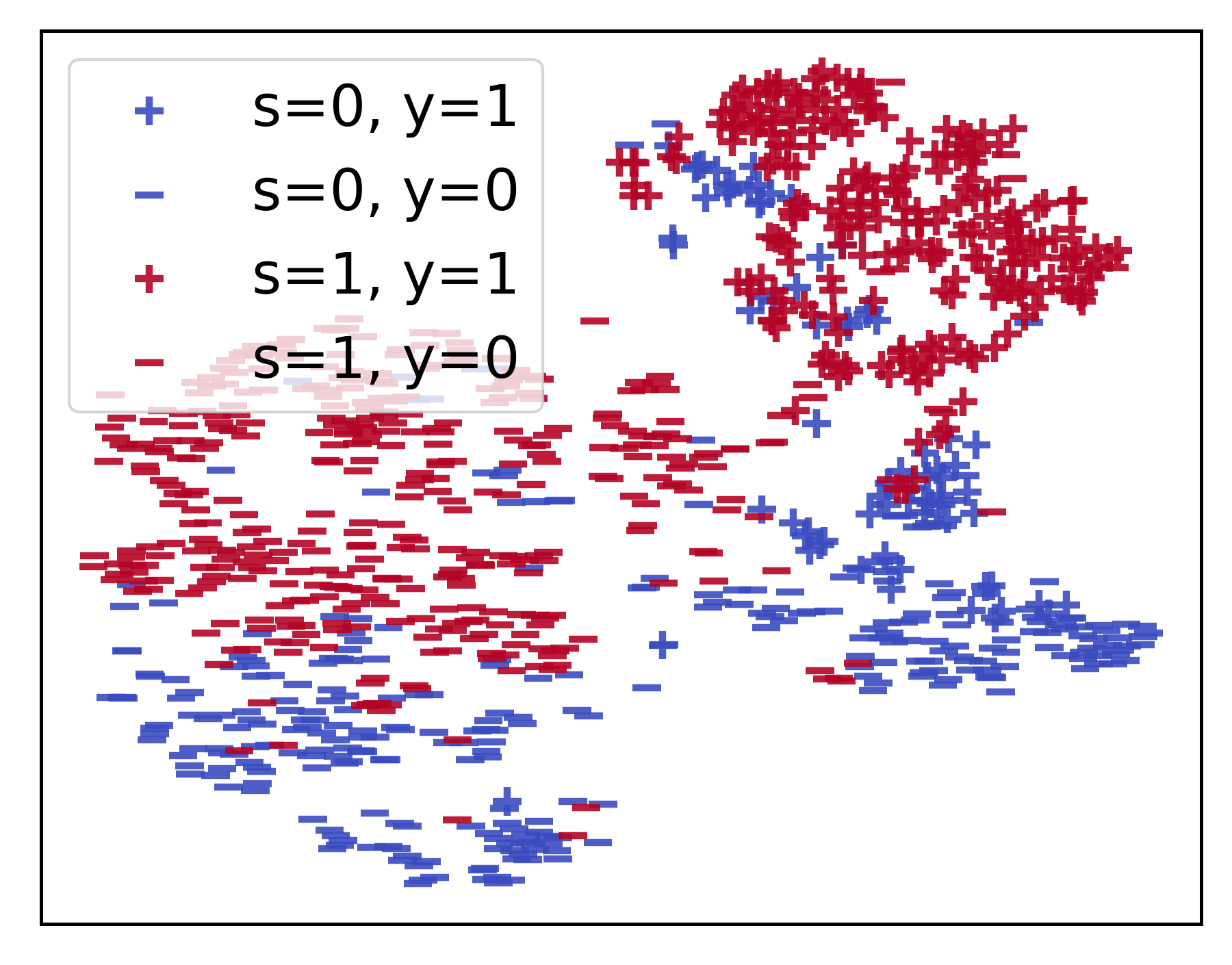}
        \hspace{0.01em}
        \includegraphics[scale=0.24]{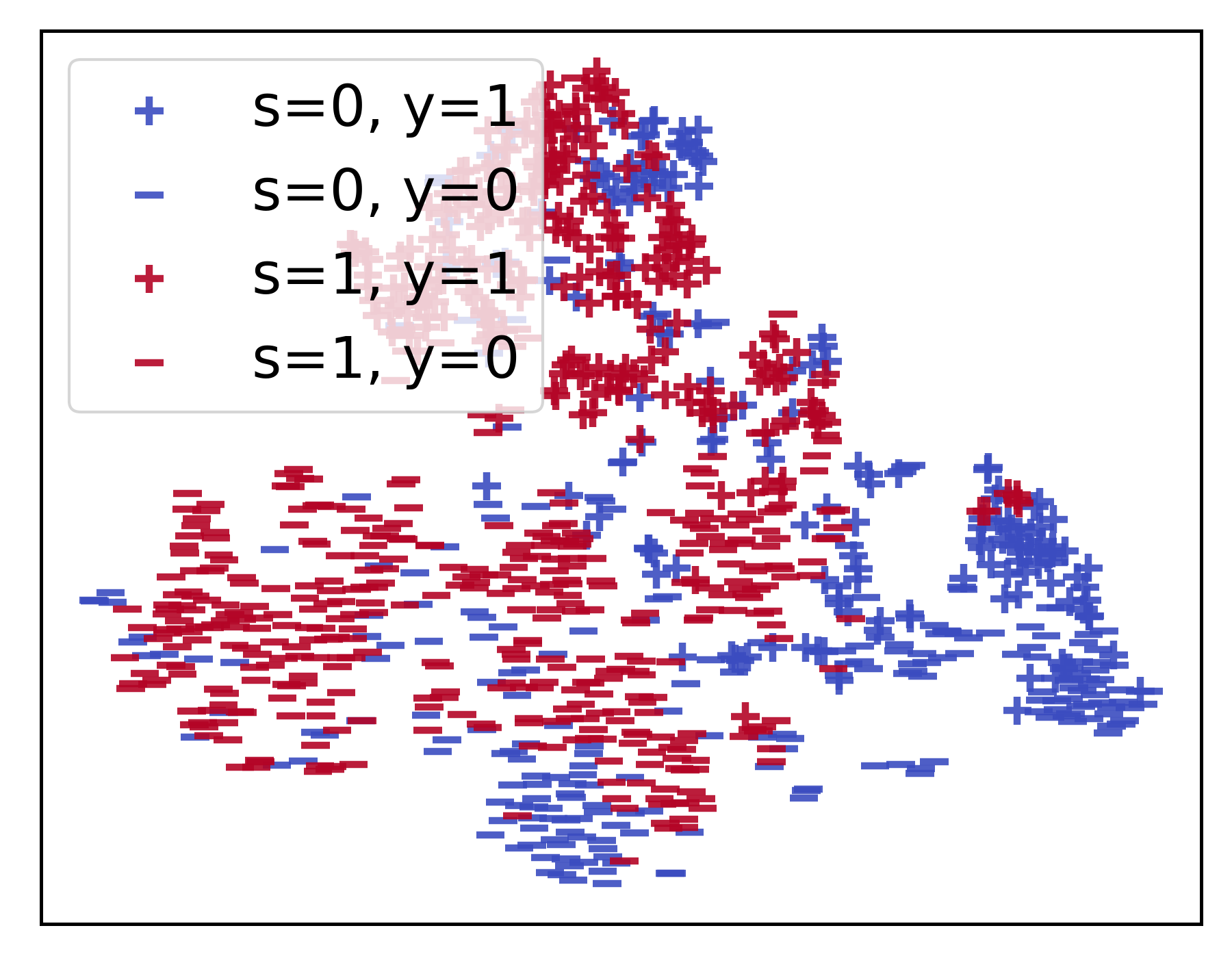}
    \end{minipage}
    }
    
    \subfloat[Taiwan Dataset]{
    \begin{minipage}[c]{0.45\textwidth}
        \centering
        \includegraphics[scale=0.24]{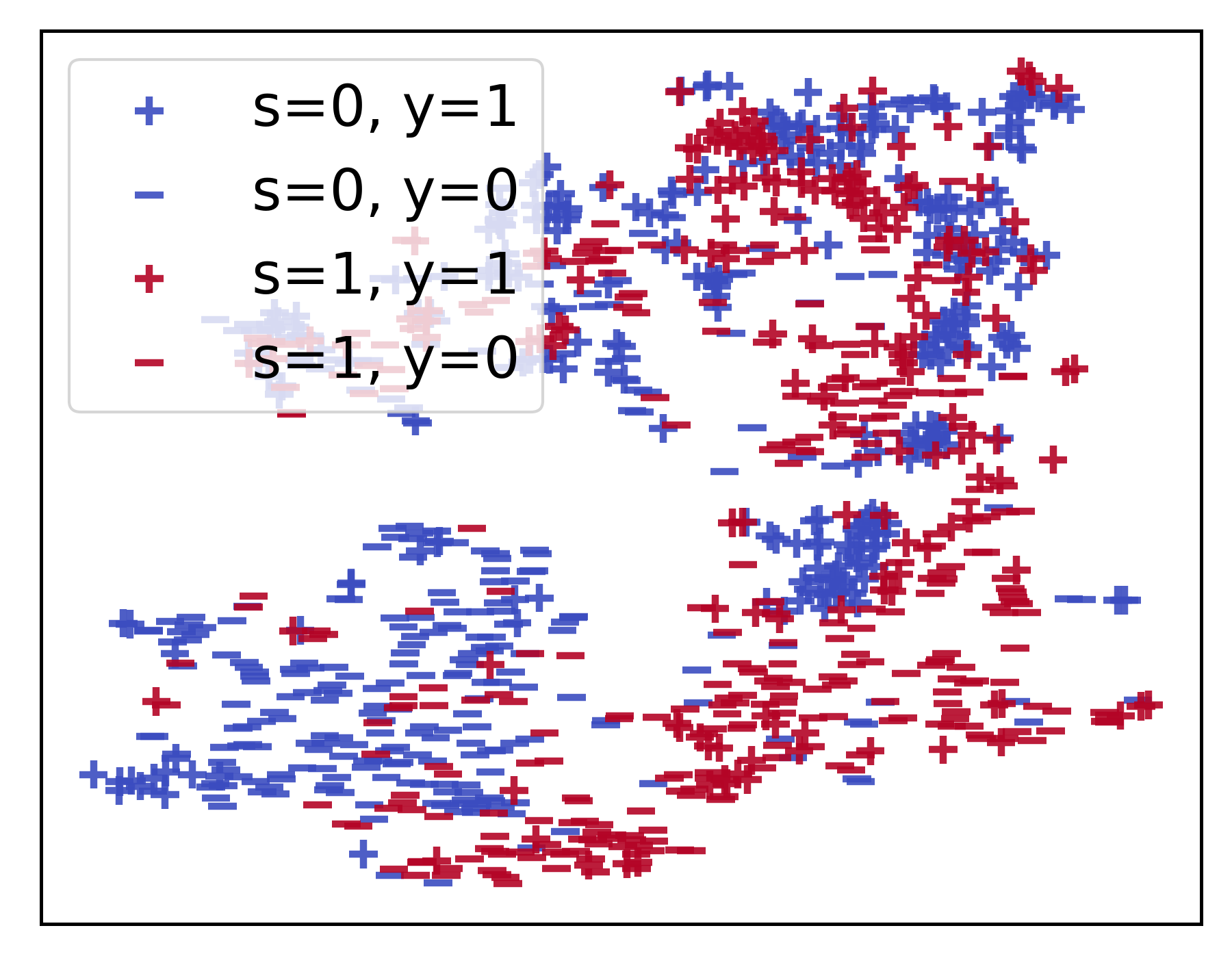}
        \hspace{0.01em}
        \includegraphics[scale=0.24]{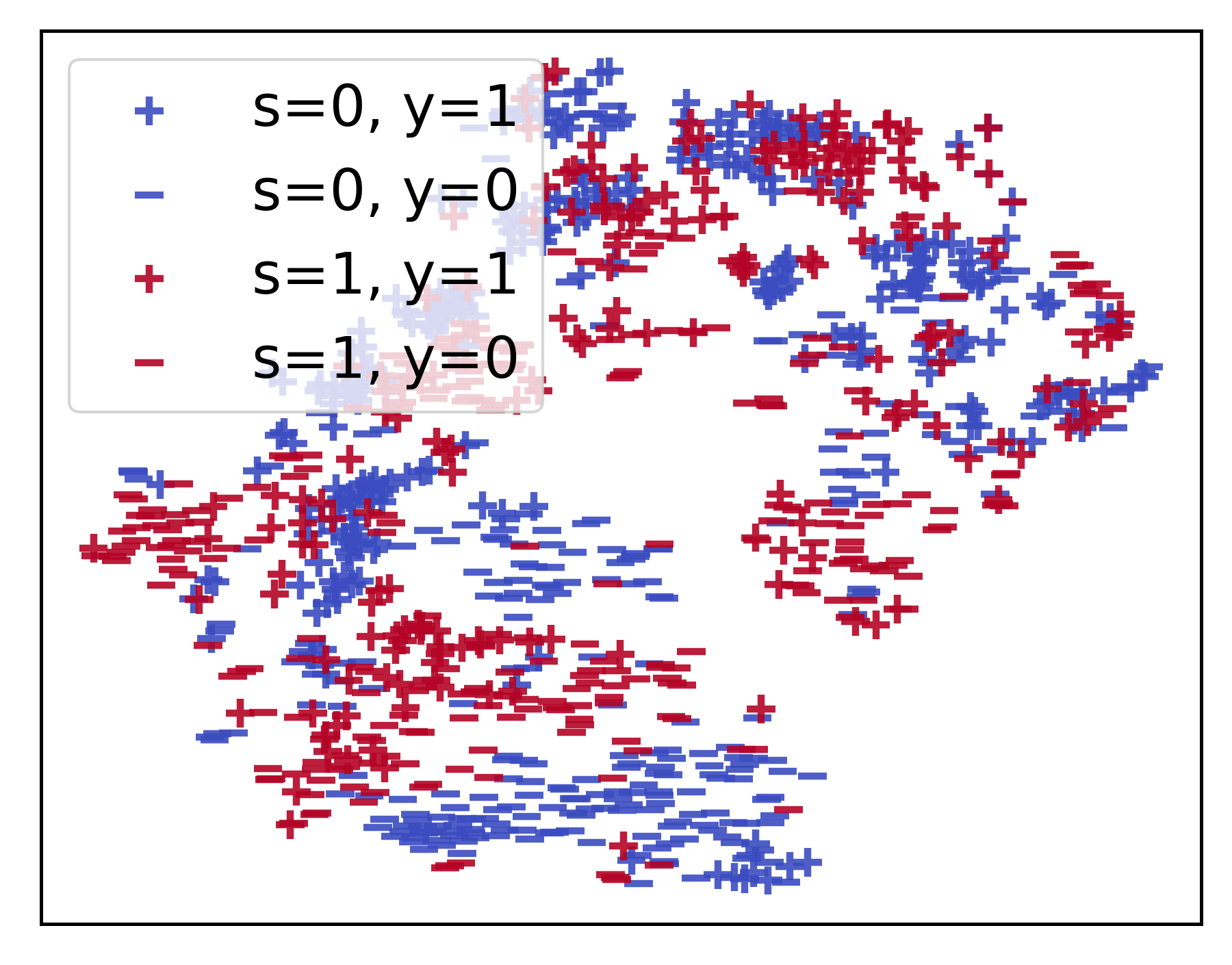}
    \end{minipage}
    }
    
    \caption{T-SNE of generated data distributions at time step $t=10$ produced by MLP (left) and DeepLF (right).} 
    \label{fig:step10}
\end{figure}

To evaluate the performance of our algorithm and baselines,  
we conduct experiments with two settings on both SimLoan and Taiwan datasets for 5 times and calculate their mean and std. In the first setting, the time step $T$ for achieving long-term fairness is set to $10$. We train the models on the 10-step training data (i.e., within the time range $[1,10]$) and evaluate the models on the 10-step generated datasets with $\textbf{X}^1$ as input (i.e., also within the time range $[1,10]$). The results of accuracy and unfairness of all algorithms on the two datasets are shown in Figures \ref{fig:result} (a) and (c). As can be seen, both the local and long-term fairness produced by \textbf{DeepLF} are comparable with or better than those of \textbf{LRLF}, and they markedly outperform other baselines. For \textbf{LRLF}, although it also produces relatively small local and long-term unfairness, it requires true causal structure equations for training which may not be available in practice. For the other three baselines, there is no clear decreasing trend in both local and long-term unfairness, although a relatively higher level of accuracy is achieved. The results demonstrate that our method strikes an effective balance between long-term fairness, local fairness, and utility, and it requires only the information from the historical data.

To illustrate the difference in the qualification distribution produced by different methods, 
we adopt T-SNE to visualize the distribution of $\X^{10}$ produced by \textbf{MLP} and \textbf{DeepLF}, as shown in Figure \ref{fig:step10}. It can be seen that compared with the distribution obtained by using the \textbf{MLP} as the decision model (left figures), the data samples of two groups ($s=0,1$) produced by \textbf{DeepLF} (right figures) are more evenly mixed together which implies a fairer qualification distribution.

In the second setting, the time step $T$ for achieving long-term fairness is set to $19$. We train the decision models on the same training data as in the first setting but evaluate the models on the 10-step generated data with $\textbf{X}^{10}$ as the input, i.e., the generated data within the time range $[10,19]$. The difference in the second setting is that we only modify the decision model that will be deployed in the future (i.e., starting from $t=10$). 
The results are shown in Figures \ref{fig:result} (c) and (d). In general, we observe similar results to the first setting where our algorithm outperforms all the baseline methods in achieving the best trade-off.

\section{Conclusions}
In this paper, we studied the problem of mitigating group disparity and achieving long-term fairness while limiting the use of sensitive attribute in decision-making. We proposed a data-driven method that requires only the information from the historical data. Leveraging the temporal causal graph, we formulated long-term fairness as the 1-Wasserstein distance between the interventional
distributions of different demographic groups. Then, we proposed a three-phase learning framework to achieve long-term fairness by training an RCGAN to predictively generate observational and interventional data and then training a classifier upon the generated data. Experiments on both synthetic and semi-synthetic data show that our method can achieve a more effective balance between long-term fairness, local fairness, and utility compared with methods based on traditional fairness notions.

\section*{Acknowledgments}
This work was supported in part by NSF 1910284, 1946391, and 2142725.

\end{document}


\maketitle

\section{Related Work}
Fair machine learning in the past decade has been focused on static settings with one-shot decisions being made \cite{mehrabi2021survey,caton2020fairness}. 
In recent years, attention has been paid to dynamic
settings where sequential decisions are made over time. 
Some efforts have been devoted to a compound decision-making process called pipeline \cite{bower2017fair,dwork2018fairness}. In pipelines, individuals may drop out at any stage, and classification in subsequent stages depends on the remaining cohort of individuals. 
For instance, hiring is at least a two-stage model: deciding whom to be interviewed from the applicant pool and then deciding whom to be hired from the interview pool.
In addition to the pipeline, a more practical and challenging dynamic setting considers that decisions made in the past can reshape the data population and subsequently influence future decisions \cite{zhang2020fairness}.
In this setting, several studies have demonstrated the inadequacy of static fairness approaches in various application scenarios, including credit lending \cite{liu2018delayed}, college admission \cite{d2020fairness}, labor market \cite{hu2018short}. 
In \cite{creager2020causal}, the authors propose to use causal directed acyclic graphs (DAGs) as a unifying framework to study fairness in dynamical systems but have not reached any approach to achieve long-term fairness. 
In \cite{hu2020fair}, the authors studied fair multiple decision making which also applies SCM and leverages soft interventions to model the deployment of decision models. However, \cite{hu2020fair} is focused on the static fairness of each decision model separately other than the long-term fairness.
As a related line of work, some research (e.g., \cite{jabbari2017fairness,zhang2020fair,wen2021algorithms,yupolicy}) studies long-term fairness in the context of reinforcement learning whose setting is different from supervised learning.
Another related line of work proposes effort-based fairness measures that balance the effort an individual needs to make to change the decision outcome between two groups \cite{heidari2019long,huan2020fairness,guldogan2022equal}. The hypothesis is that effort fairness will encourage rejected individuals to improve their qualifications and prevent the exacerbation of the gap between different groups in the long run.
For example, in \cite{heidari2019long}, the authors propose a framework for characterizing the long-term impact of decision making algorithms on reshaping the distribution and leverage social models to simulate how individuals may respond to the decisions. The most relevant work to this paper is \cite{hu2022achieving}. It studied long-term fair decision making and formulated long-term fairness from the causal perspective. However, \cite{hu2022achieving} requires true causal structure equations for training. In addition, it cannot achieve fairness at a time step that is beyond the training data. These limitations greatly reduce its practical significance.

\section{Proof of Proposition 1}

\begin{proposition}\label{pro:wd}
Let $d$ be the 1-Wasserstein distance given in Definition 1.

For any sensitive attribute-unconscious decision model $f: \mathcal{X} \mapsto \mathcal{A}$ that is Lipschitz continuous, its DP is bounded by $l_f \cdot d$ where $l_f$ is the Lipschitz constant of $f$. If we assume that the true label $Y$ is given by a decision model $g: \mathcal{X} \mapsto \mathcal{A}$ that is Lipschitz continuous and satisfies the equal base rate condition, then the EO of $f$ is bounded by $(l_f+l_g)/P(y)\cdot d$ where $l_g$ is the Lipschitz constant of $g$.
\end{proposition}

\begin{proof}
For simplicity, in this proof we drop the superscript $T$ and the notation of the soft intervention for $\X^T(\sigma_{\theta})$.
According to the definition of DP, we have
\begin{equation*}
    \mathrm{DP}(f) = |\mathbb{E}[f(\mathbf{X})|S=s^+] - \mathbb{E}(f(\mathbf{X})|S=s^-)|.
\end{equation*}
Due to the Kantorovich–Rubinstein duality \cite{villani2021topics}, it is straightforward that
\begin{equation*}
\begin{split}
    \mathrm{DP}(f) & \leq \sup_{\lVert f \rVert \leq l_f} \left[ \mathbb{E}_{\x\sim P(\x|s^+)}[f(\x)] - \mathbb{E}_{\x\sim P(\x|s^-)}[f(\x)] \right] \\
    & = l_f \cdot W(P(\X|S=s^+),P(\X|S=s^-)) = l_f \cdot d.
\end{split}
\end{equation*}
On the other hand, we have
\begin{equation*}
    \mathrm{EO}(f) = |\mathbb{E}[f(\X)|Y\!=\!1, S\!=\!s^+] - \mathbb{E}(f(\X)|Y\!=\!1, S\!=\!s^-)|.
\end{equation*}
Since we assume that $Y$ is given by $g$ and $g$ satisfies the equal base rate condition, we have that $P(Y|\X,S)=g(\X)$ and $P(Y|S)=P(Y)$. It then follows that
\begin{equation*}
\begin{split}
    & \mathbb{E}[f(\X)|y, s] = \int_{\x} f(\x)P(\x|y,s) d\x \\
    & = \int_{\x} f(\x)P(\x|s)\frac{P(y|\x,s)}{P(y|s)} d\x = \int_{\x} f(\x)P(\x|s)\frac{g(\x)}{P(y)} d\x \\
    & = \frac{1}{P(y)}\mathbb{E}_{\x\sim P(\x|s)}[f(\x)g(\x)].
\end{split}
\end{equation*}
In addition, define $m(x)=f(x)g(x)$ and denote its Lipschitz constant as $l_m$. It is easy to show that $l_m \leq l_f \cdot \sup_{\X} |f(\X)| + l_g \cdot \sup_{\X} |g(\X)|$. Since $h(\X)\leq 1$ and $g(\X)\leq 1$, we have $l_m\leq l_f+l_g$. As a result, we have
\begin{equation*}
\begin{split}
    EO(f) & \leq \frac{l_f+l_g}{P(y)} W(P(\X|S=s^+),P(\X|S=s^-))\\
    & = \frac{l_f+l_g}{P(y)} \cdot d.
\end{split}
\end{equation*}

\end{proof}

\section{Implementations Details and Hyperparameters}
Experiments are performed on the computer with Intel Core i7-9700K CPU and NVIDIA GeForce GTX 1180 GPU. Except for \textbf{LRLF}, other baselines and our framework are used multi-layer fully-connected networks, i.e., \textbf{MLP}, as the classifiers. The details of the model architectures and hyperparameters used in our framework on two datasets are given in Tables \ref{tab:1} and \ref{tab:2}. For a fair comparison, we adopt the same network structure and parameter settings for our decision model $h_{\theta}$.
Both datasets are split into train/validation/test sets with the ratio 70/10/20. The models are trained on the train sets and the hyperparameters are chosen on the validation sets. The reported results are calculated on the test sets.

\section{Data Generation}
\subsubsection{Synthetic Dataset.} We generate the synthetic time series dataset based on the causal time series graph shown in Figure 1 in the main paper.
Each sample at each time step in the time series includes a sensitive feature $S$, profile features $\textbf{X}^t$ and a decision $Y^t$. The samples at the initial time step $\textbf{X}^1,Y^{1}$ are generated by calling the data generation function (i.e., make\_classification) of scikit-learn package. Then, we cluster the generated samples into two groups and assign $S$ to each sample according to the cluster it belongs to. To generate the data samples in the remaining time steps, we design a procedure by simulating the bank loan system in the real world. We first train a neural network classifier $h_{\theta^*}$ on $S,\textbf{X}^1,Y^{1}$ and treat it as the ground-truth model. For each time step $t$, classifier $h_{\theta^*}$ takes as inputs $S$ and $\textbf{X}^t$ and outputs a probability distribution over $Y^{t}$. We then sample $Y^{t}$ from the distribution as shown below:
\begin{equation}
    \begin{split}
    P(Y^t) = h_{\theta^*}(S, \textbf{X}^t) \quad\quad Y^t \sim \text{Bernoulli}(P(Y^t))
    \end{split}
\end{equation}
After that, we update the value of $\X^{t}$ to obtain $\X^{t+1}$ based on the value of $Y^{t}$. We treat $Y^{t}$ as the ground-truth of loan repayment ($Y^{t}=1$) and default ($Y^{t}=0$). An individual with $Y^t=1$ should have a larger probability to be predicted as $1$ in the next time step, and vice versa. Therefore, we update the value of $\X^{t}$ according to the value of $Y^{t}$ as well as the gradient of a loss function between the predicted probability and label 1, as given below:
\begin{equation}
    \begin{split}
    &\textbf{X}^{t+1} = \textbf{X}^{t} - \epsilon\cdot (2Y^t-1) \cdot \frac{\partial \mathcal{L}(h_{\theta^*}(S,\X^t), \textbf{1})}{\partial \textbf{X}^t} \\
    \end{split}
\end{equation}
where the parameter $\epsilon$ controls the magnitude of changes in $\textbf{X}^t$. As a result, $\X^{t+1}$ will be predicted closer to label 1 if $Y^t=1$, and will be predicted further from label 1 if $Y^t=0$.
Following above generation rules, we generate a 10-step synthetic time series dataset with 10000 instances and $\textbf{X}^t$ is 6 dimensional vector. We refer to this dataset SimLoan.

\subsubsection{Semi-Synthetic Dataset.}
We also generage semi-synthetic data by leveraging the real-world Taiwan dataset as the initial data at $t=1$. A ground-truth classifier and similar generation rules of change are used to generate subsequent decisions $Y^1, ..., Y^l$ and profile features $\textbf{X}^2, ..., \textbf{X}^l$. There are 10000 instances in the initial data and they are randomly and equally sampled from groups by $S$ and $Y$ for balance. Like the SimLoan dataset, this dataset is also made up of 10 steps. We choose SEX as the sensitive feature $S$ and BILL\_ATM1 - BILL\_ATM6 as the profile features in $\textbf{X}$. We refer to this dataset Taiwan.

\begin{table}[h]
\renewcommand\arraystretch{1.2}
\caption{The architectures of CLF and $h_\theta$ and hyperparameters for both datasets}
\centering
\label{tab:1}
\begin{tabular}{cccc}
\hline
\multirow{2}{*}{Layer} & \multirow{2}{*}{Inputs} & \multicolumn{2}{c}{Output Dim} \\ \cline{3-4}
                       &                         & SimLoan   & Taiwan   \\ \hline
X                      &                         & 6           & 6                \\
S                      &                         & 1           & 1                \\
FC\_1                  & {[}X, S{]}              & 32          & 16               \\
FC\_2                  & FC\_1                   & 64          & 32               \\
FC\_3                  & FC\_2                   & 1           & 1                \\ \hline
Optimizer              & Adam                    &             &                  \\
Learning rate          & 0.001                   &             &                  \\
Batch size             & 512                     &             &                  \\
$\lambda_u$            &                         & 1.0         & 1.0              \\
$\lambda_s$            &                         & 2.1         & 0.2              \\
$\lambda_l$            &                         & 128.4       & 40.0             \\\hline
\end{tabular}
\end{table}

\begin{table}[h]
\renewcommand\arraystretch{1.2}
\caption{The architecture of RCGAN and hyperparameters for both datasets}
\centering
\label{tab:2}
\begin{tabular}{cccc}
\hline
\multirow{2}{*}{Layer} & \multirow{2}{*}{Inputs} & \multicolumn{2}{c}{Output Dim} \\ \cline{3-4} 
                       &                         & SimLoan   & Taiwan   \\ \hline
X/Z             &                         & 6           & 6                \\
S/Y             &                         & 1           & 1                \\ \hline
Generator     &                         &             &                  \\
GRU\_1        & {[}Z, S, Y{]}       & 64          & 64               \\
GRU\_2        & GRU\_1                  & 64          & 64               \\
FC\_1         & GRU\_2                  & 6           & 6                \\
Penalty       &                         &             &                  \\
MMD           & {[}X, FC1{]}            & 1           & 1                \\ \hline
Discriminator &                         &             &                  \\
GRU\_1        & FC\_1                   & 64          & 64               \\
GRU\_2        & GRU1                    & 64          & 64               \\
FC\_1         & GUR\_2                  & 1           & 1                \\ \hline
Opimizer      & Adam                    &             &                  \\
Learning rate & 0.001                   &             &                  \\
Batch size    & 512                     &             &                  \\
$\gamma$         & 100                     &             &                  \\ \hline
\end{tabular}
\end{table}